\documentclass{article}

\usepackage{graphicx} 
\usepackage{subfigure} 

\usepackage{natbib}
\usepackage{amsmath}
\usepackage{amssymb}
\usepackage{amsthm}

\usepackage{algorithm}
\usepackage{algorithmic}

\usepackage{hyperref}

\newtheorem{theorem}{Theorem}[section]

\newtheorem{lemma}[theorem]{Lemma}
\newtheorem{definition}[theorem]{Definition}
\newtheorem{remark}[theorem]{Remark}

\usepackage{multirow}

\newcommand{\norm}[1]{\left\Vert#1\right\Vert}
\newcommand{\abs}[1]{|#1|}
\newcommand{\set}[1]{\left\{#1\right\}}

\newcommand{\kron}{\otimes}
\newcommand{\dom}[1]{\mathrm{dom}(#1)}
\renewcommand{\vec}{\mathbf{vec}}
\newcommand{\mat}{\mathbf{mat}}
\newcommand{\trace}{\mathrm{tr}}
\newcommand{\cov}{\boldsymbol{\Sigma}}
\newcommand{\invcov}{\boldsymbol{\Sigma}^{-1}}
\usepackage[accepted]{icml2013} 

\newcommand{\vectornorm}[1]{\|#1\|}

\newcommand{\T}{{\bf \Theta}}

\newcommand{\tr}{\text{tr}}

\DeclareMathAlphabet{\mathbbb}{U}{bbold}{m}{n}

\DeclareMathOperator*{\argmin}{arg\,min}

\icmltitlerunning{A proximal Newton framework for composite minimization}

\begin{document} 

\twocolumn[
\icmltitle{A proximal Newton framework for composite minimization: Graph learning without Cholesky decompositions and matrix inversions}


\icmlauthor{Quoc Tran Dinh}{quoc.trandinh@epfl.ch}
\icmlauthor{Anastasios Kyrillidis}{anastasios.kyrillidis@epfl.ch}
\icmlauthor{Volkan Cevher}{volkan.cevher@epfl.ch}
\icmladdress{LIONS, \'{E}cole Polytechnique F\'{e}d\'{e}rale de Lausanne, Switzerland}
%

\icmlkeywords{Graph selection, sparse covariance estimation, proximal-Newton method, structured convex optimization}

\vskip 0.3in
]
\begin{abstract}
We propose an algorithmic framework for convex minimization problems of composite functions with two terms:  a self-concordant part and a possibly nonsmooth regularization part.  Our method is a new proximal Newton algorithm with local quadratic convergence rate. As a specific problem instance, we consider sparse precision matrix estimation problems in graph learning. Via a careful dual formulation and a novel analytic step-size selection, we instantiate an algorithm within our framework for graph learning that avoids {\it Cholesky decompositions and matrix inversions}, making it attractive for parallel and distributed implementations. 
\end{abstract}
\section{Introduction}\label{sec:intro}
Sparse inverse covariance matrix estimation is a key step in graph learning. To understand the setup, let us consider learning a Gaussian Markov
random field (GMRF) of $p$ nodes/variables from a dataset $\mathcal{D} := \set{{\bf x}_{1}, {\bf x}_{2}, \dots, {\bf x}_{m}}$, where ${\bf x}_j \in
\mathcal{D}$ is a $p$-dimensional random vector, drawn from the Gaussian distribution $\mathcal{N}(\boldsymbol{\mu}, \cov)$. Let $\T = \invcov$ be the inverse covariance
(or the precision) matrix for the model. To satisfy the conditional independencies with respect to the GMRF, $\T$ must have zero in $\T_{ij}$
corresponding to the absence of an edge between nodes $i$ and $j$ \cite{Dempster1972}.

To learn the underlying graph structure from $\invcov$, one can use the empirical covariance matrix $\widehat{\cov}$. Unfortunately, this approach is 
fundamentally ill-posed since the empirical estimates converge to the true covariance at a $(1/\sqrt{m})$-rate \cite{Dempster1972}. Hence, inferring the true
graph structure accurately requires an overwhelming number of samples. Unsurprisingly, we usually have fewer samples than the ambient dimension, compounding the
difficulty of estimation. 

While the possible GMRF structures are exponentially large, the most interesting graphs are rather simple with a sparse set of edges. Provable
learning of such graphs can be achieved by $\ell_1$-norm regularization in the maximum log-likelihood estimation:
\begin{equation}\label{eq:glearn_prob}
\!\!\!\!\small{\T^{\ast} \!\in\! \argmin_{\T \succ 0} \Big\{\!\underbrace{-\log \det (\T) \!+\! \tr(\widehat{\cov} \T)}_{=:f(\T)} \!+\! \underbrace{\rho
\vectornorm{\vec(\T)}_1}_{=:g(\T)}
\Big \}}, 
\end{equation} 
where $\rho > 0$ is a parameter to balance the fidelity error and the sparsity of the solution and $\vec$ is the vectorization operator. 
Here, $f(\T)$ corresponds to the empirical log-likelihood and $g(\T)$ is the sparsity-promoting term.  
Under this setting, the authors in \cite{Ravikumar2011} prove that $m = \mathcal{O}(d^2\log p)$ is sufficient for correctly estimating the
GMRF, where $d$ is the graph node-degree. Moreover, the above formulation still makes sense for learning other graph models, such as the Ising model, due to the
connection of $f(\T)$ to the Bregman distance \cite{Banerjee2008}. 

Numerical solution methods for solving problem \eqref{eq:glearn_prob} have been widely studied in the literature. 
For instance, in \cite{Banerjee2008, Scheinberg2009, Scheinberg2010, Hsieh2011, Rolfs2012, Olsen2012} the authors proposed first order primal and dual approaches to
\eqref{eq:glearn_prob} and used state-of-the-art structural convex optimization techniques such as coordinate descent methods and Lasso-based procedures.
Alternatively, the authors in \cite{Hsieh2011, Olsen2012} focused on the second order methods and, practically, achieved fast methods with a high accuracy. 
In \cite{Scheinberg2010,Yuan2012}, the authors studied alternating direction methods to solve \eqref{eq:glearn_prob}, while the work in \cite{Li2010} is based on
interior point-type methods. 
Algorithmic approaches where more structure is known a priori can be found in \cite{Lu2010}. 

The complexity of the state-of-the-art approaches mentioned above is dominated by the Cholesky decomposition ($\mathcal{O}(p^3)$ in general), 
which currently creates an important scalability bottleneck. This decomposition appears mandatory since all these approaches employ a guess-and-check step-size
selection procedures to ensure the iterates remain in the positive definite (PD) cone and the inversion of a $p\times p$ matrix, whose theoretical cost normally
scales with the cost of $p\times p$ matrix multiplications ($\mathcal{O}(p^3)$ direct, $\mathcal{O}(p^{2.807})$ Strassen, and $\mathcal{O}(p^{2.376})$
Coppersmith-Winograd). The inversion operation is seemingly mandatory in the optimization of \eqref{eq:glearn_prob} since the calculation of the descent
direction $\nabla f(\T_i) := -\T_i^{-1} + \widehat{\cov}$ requires it, and quadratic cost approximations to $f(\T)$ also need it. Via Cholesky decompositions,
one can first check if the current solution satisfies the PD cone constraint and then recycle the 
decomposition for inversion for the next iteration.

\noindent\textbf{Contributions:}
We propose a new proximal-Newton framework for solving the general problem of \eqref{eq:glearn_prob} by only assuming that $f(\cdot)$ is self-concordant.  
Our algorithm consists of two phases. 
In Phase 1, we apply a damped proximal Newton scheme with a new, analytic step-size selection procedure, and prove that
our objective function always decreases at least a certain fixed amount. As a result, we avoid globalization strategies such as backtracking line-search
or trust-region procedures in the existing methods.
Moreover, our step-size selection is optimal in the sense that it cannot be improved without additional assumptions on the problem structure. 
In Phase 2, we simply apply the full step proximal-Newton iteration as we get into its provable quadratic convergence  region which we can compute explicitly.
Moreover, we do not require any additional assumption such as the uniform boundedness of the Hessian as in \cite{Lee2012}.

In the context of graph learning, we discuss a specific instance of our framework, which avoids Cholesky decompositions and matrix inversions
\emph{altogether}. 
Hence, the per iteration complexity of our approach is dominated by the cost of $p\times p$ matrix multiplications.  This is because $(i)$ our analytic
step-size selection procedure ensures the positive definiteness of the iterates doing away with \textit{global strategies} such as line-search which demands
the objective evaluations (via Cholesky), and $(ii)$ we avoid calculating the gradient explicitly, and hence matrix inversion by a careful dual formulation. 
As a result, our approach is attractive for distributed and parallel implementations.
 
\textbf{Paper outline:} 
In Section \ref{sec:preliminary}, we first recall some fundamental concepts of convex optimization and self-concordant functions. Then, we describe the basic
optimization set up and show the unique solvability of the problem. 
In Section \ref{sec:prob_setting} we outline our algorithmic framework and describe its analytical complexity.
We also deal with the solution of the subproblems by applying the new dual approach in this section. 
Section \ref{sec:app_graph_select} presents an application of our theory to graph selection problems.
Experimental results on real graph learning problems can be found in Section \ref{sec:num_test}. 

\section{Preliminaries}\label{sec:preliminary}
\vskip-0.2cm
\noindent\textbf{Basic definitions:}
We reserve lower-case and bold lower-case letters for scalar and vector representation, respectively. Upper-case bold letters denote matrices. Let $\vec$:
$\mathbb{R}^{p \times p} \rightarrow \mathbb{R}^{p^2}$ be the vectorization operator which maps a matrix to a single column, and $\mat$: $\mathbb{R}^{p^2}
\rightarrow \mathbb{R}^{p \times p}$ is the inverse mapping of $\vec$ which transforms a vector to a matrix. For a closed convex function $f$, we denote its
domain by $\dom{f}$, $\dom{f} := \set{x\in\mathbb{R}^n~|~ f(x) < +\infty}$.

\begin{definition}[Self-concordant functions (Definition 2.1.1, pp. 12, \cite{Nesterov1994}]\label{concordant}
A convex function $h: \mathbb{R} \rightarrow \mathbb{R} $ is $($standard$)$ self-concordant if $\abs{h'''(x)} \leq 2h''(x)^{3/2}, ~\forall x \in \mathbb{R}$. 
Furthermore, a function $h: \mathbb{R}^{n} \rightarrow \mathbb{R}$ is self-concordant if, for any $t \in \mathbb{R}$, the function $\phi(t) := h({\bf x} + t
{\bf v})$ is self-concordant for all $\mathbf{x}\in\dom{f}$ and ${\bf v} \in \mathbb{R}^{n}$. 
\end{definition}
Let $h \in \mathcal{C}^3(\dom{f})$ be a strictly convex and self-concordant function. 
For a given vector ${\bf v} \in \mathbb{R}^{n}$, the local norm around ${\bf x} \in \dom{f}$ with respect to $h(\cdot)$ is defined as $\norm{\bf
v}_{\bf x} := \left({\bf v}^T\nabla^2h({\bf x}){\bf v}\right)^{1/2}$ while the corresponding dual norm is given as $\norm{\bf v}_{\bf x}^{*} := \left({\bf v}^T\nabla^2 h({\bf x})^{-1}{\bf v}\right)^{1/2}$. 
Let $\omega : \mathbb{R}\to\mathbb{R}_{+}$ be a function
defined as $\omega(t) := t - \ln(1+t)$ and $\omega_{*} : [0,1]\to\mathbb{R}_{+}$ be a function defined as $\omega_{*}(t) := -t - \ln(1-t)$. The functions
$\omega$ and $\omega_{*}$ are both nonnegative, strictly convex and increasing.  Based on \cite{Nesterov2004}[Theorems 4.1.7~\&~4.1.8], we recall the following estimates:
\begin{eqnarray}\label{eq:SC_bounds}
&\omega(\norm{{\bf y} - {\bf x}}_{\bf x}) + \nabla{h}({\bf x})^T({\bf y} - {\bf x}) + h({\bf x}) \leq h({\bf y}),\label{eq:SC_bound1}\\
&h({\bf y}) \leq h({\bf x}) + \nabla{h}({\bf x})^T({\bf y} - {\bf x}) + \omega_{*}(\norm{{\bf y} - {\bf x}}_{\bf x}),\label{eq:SC_bound2}
\end{eqnarray}
where \eqref{eq:SC_bound1} holds for all ${\bf x}, {\bf y}\in\dom{f}$, and \eqref{eq:SC_bound2} holds for all ${\bf x}, {\bf y}\in\dom{f}$ such that $\norm{{\bf
y} - {\bf x}}_{\bf x} < 1$.

\noindent\textbf{Problem statement:}
In this paper, we consider the following structural convex optimization problem:
\begin{align}\label{eq:min_Fx}
\min_{{\bf x} \in \mathbb{R}^n} \Big \{ F({\bf x})~~|~~ F({\bf x}) := f({\bf x}) + g({\bf x}) \Big\},
\end{align} 
where $f({\bf x})$ is a {\it convex, self-concordant} function and $g({\bf x})$ is a proper, lower semicontinuous and possibly nonsmooth convex regularization term. 
It is easy to certify that problem \eqref{eq:glearn_prob} can be transformed into \eqref{eq:min_Fx} by using the tranformation ${\bf x} := \vec(\T)$:
\begin{equation*}
\footnotesize{f({\bf x}) := \left\{\begin{array}{ll}
    \!\!-\log\det(\mat({\bf x})) \!+\! \trace(\widehat{\cov}\mat({\bf x})), &\!\!\! \mat({\bf x})\succ 0,\\
    \!\!+\infty &\!\!\!\textrm{otherwise}, \\
\end{array} \nonumber
\right.}
\end{equation*} 
$g({\bf x}) := \rho\norm{{\bf x}}_1$ and $n := p^2$.

\noindent\textbf{Proximity operator:}  
A basic tool to handle nonsmooth convex functions is the proximity operator: let $g$ be a proper lower semicontinuous, possibly nonsmooth and convex in $\mathbb{R}^n$. We denote by $\partial{g}({\bf x})$ the subdifferential of
$g$ at ${\bf x}$. Let $f$ be a self-concordant function and ${\bf x}\in\mathrm{dom}(f)$ be fixed. 
We define $P_g^{\bar{{\bf x}}}({\bf u}) := (\nabla^2f(\bar{{\bf x}}) + \partial{g})^{-1}({\bf u}) $ for ${\bf u} \in \mathbb{R}^n$. 
This operator is a nonexpansive mapping, i.e.,
\begin{equation}\label{eq:P_g_property2}
\norm{P_g^{{\bar{\bf x}}}({\bf u}) - P_g^{{\bar{\bf x}}}({\bf v})}_{{\bf x}} \leq \norm{{\bf u}-{\bf v}}^{*}_{{\bf x}},~~\forall \mathbf{u}, \mathbf{v}.
\end{equation}

\noindent\textbf{Unique solvability of the problem:}
We generalize the result in \cite{Hsieh2011} to show that problem \eqref{eq:min_Fx} is uniquely solvable.

\begin{lemma}\label{le:unique_solution}
For some ${\bf x}\in\dom{F}$, let $\lambda({\bf x}) := \norm{\nabla{f}({\bf x}) + {\bf v}}_{\bf x}^{*} < 1$ for ${\bf v}\in\partial{g}({\bf x})$. 
Then the solution ${\bf x}^{*}$ of \eqref{eq:min_Fx} exists and is unique. 
\end{lemma}
The proof of this lemma can be done similarly as Theorem 4.1.11, pp. 187 in \cite{Nesterov2004}. 
For completeness, we provide it in the supplementary document.

\section{Two-phase proximal Newton method}\label{sec:prob_setting}
Our algorithmic framework is simply a proximal-Newton method which generates an iterative sequence $\set{{\bf x}^k}_{k\geq 0}$ starting from ${\bf x}^0\in
\dom{F}$. The
new point ${\bf x}^{k+1}$ is computed as ${\bf x}^{k+1} = {\bf x}^k + \alpha_k{\bf d}^k$, where $\alpha_k\in (0, 1]$ is a step size and ${\bf d}^k$ is the
proximal-Newton-type direction as the solution to the subproblem: 
\begin{equation}\label{eq:cvx_subprob1}
\min_{{\bf d}} \left\{ Q({\bf d}; {\bf x}^k) + g({\bf x}^k + {\bf d}) \right\}.\tag{$\mathcal{Q}(\mathbf{x}^k)$}
\end{equation}
Here, $Q({\bf d}; {\bf x}^k)$ is the following quadratic surrogate of the function $f$ around ${\bf x}^k$:
\begin{align}\label{eq:Q_x}
Q({\bf d}; {\bf x}^k) := f({\bf x}^k) \!+\! \nabla{f}({\bf x}^k)^T{\bf d} \!+\! \frac{1}{2}{\bf d}^T\nabla^2{f}({\bf x}^k){\bf d}. 
\end{align}
We denote $\mathbf{d}^k$ the unique solution of \ref{eq:cvx_subprob1}.
The optimality condition for \ref{eq:cvx_subprob1} is written as follows:
\begin{equation}\label{eq:cvx_subprob_optimality}
{\bf 0} \in \partial{g}({\bf x}^k + {\bf d}^k) + \nabla{f}({\bf x}^k) + \nabla^2f({\bf x}^k){{\bf d}^k}.
\end{equation}

\noindent\textbf{Fixed-point characterization.} For given $\bf{x} \in\dom{F}$, if we define $S({\bf x}) := \nabla^2f({\bf x}){\bf x} - \nabla{f}({\bf x})$ then
the unique solution ${\bf d}^k$ of \ref{eq:cvx_subprob1} can be computed as
\begin{equation}\label{eq:search_dir}
{\bf d}^k := \left(P_g^{{\bf x}^k}\circ S\right)({\bf x}^k) - {\bf x}^k = - (\mathbb{I} - R_g)({\bf x}^k).
\end{equation}
Here, $R_g(\cdot) := \left(P_g^{{\bf x}}\circ S\right)(\cdot) \equiv P_g^{{\bf x}}(S(\cdot))$.
The next lemma shows that the fixed point of $R_g$ is the unique solution of \eqref{eq:min_Fx}. The proof is straightforward, and
is omitted.

\begin{lemma}\label{le:fixed_point}
Let $R_g$ be a mapping defined by \eqref{eq:search_dir}. Then ${\bf x}^{\ast}$ is the unique solution of \eqref{eq:min_Fx} if and only if ${\bf x}^{\ast}$ is
the fixed-point
of $R_g$, i.e., ${\bf x}^{\ast} = R_g({\bf x}^{\ast})$.
\end{lemma}
Lemma \ref{le:fixed_point} suggests that we can generate an iterative sequence based on the fixed-point principle. 
Under certain assumptions, one can ensure that $R_g$ is contractive and the sequence
generated by this scheme is convergent. Hence, we characterize this below.

\subsection{Full-step proximal-Newton scheme}\label{subsec:FPN_sec}
Here, we show that if we start sufficiently close to the solution ${\bf x}^{\ast}$, then we can compute the next iteration $\mathbf{x}^{k+1}$
with full-step $\alpha_{k+1} = 1$, i.e.,
\begin{equation}\label{eq:FPNM}
{\bf x}^{k+1} := {\bf{x}}^k + \mathbf{d}^k,
\end{equation}
where $\mathbf{d}^k$ is the unique solution to \ref{eq:cvx_subprob1}.
We call this scheme the \textbf{full-step proximal Newton} (FPN) scheme.
For any $k\geq 0$, let us define 
\begin{equation}\label{eq:PNT_decrement}
\lambda_k := \norm{\mathbf{x}^{k+1} - \mathbf{x}^k}_{\mathbf{x}^k}.
\end{equation}
We refer to this quantity as the \textit{proximal Newton decrement}.
The following theorem establishes the local quadratic convergence of the FPN scheme \eqref{eq:FPNM}.

\begin{theorem}\label{th:quad_converg_FPNM}
For a given $\mathbf{x}^k$, let ${\bf x}^{k+1}$ be the point generated by the full-step proximal Newton scheme \eqref{eq:FPNM} and $\lambda_k$ be defined by
\eqref{eq:PNT_decrement}. Then, if $\lambda_k <
1-\frac{1}{\sqrt{2}}\approx 0.292893$, it holds that 
\begin{equation}\label{eq:FPNM_estimate}
\lambda_{k+1} \leq (1 - 4\lambda_k + 2\lambda^2_k)^{-1}\lambda^2_k.
\end{equation}
Consequently, the sequence $\set{{\bf x}^k}_{k\geq 0}$ generated by the FPN scheme \eqref{eq:FPNM} starting from $\mathbf{x}^0\in\dom{F}$ such that $\lambda_0
\leq \sigma \leq \bar{\sigma} := \frac{5-\sqrt{17}}{4} \approx 0.219224$, locally converges to the unique solution of \eqref{eq:min_Fx}
at a quadratic rate.
\end{theorem}
The proof of Theorem \ref{th:quad_converg_FPNM} can be found in the supplementary document.

\subsection{Damped proximal Newton scheme}\label{subsec:damped_newton_step}
We now establish  that, with an appropriate choice of the step-size $\alpha \in (0,1]$, the iterative sequence $\set{{\bf x}^k}_{k\geq 0}$
generated by the \textit{damped proximal Newton} scheme
\begin{equation}\label{eq:damped_PNM}
\mathbf{x}^{k+1} := \mathbf{x}^k + \alpha_k\mathbf{d}^k
\end{equation}
is a decreasing sequence, i.e., $F({\bf x}^{k+1}) \leq F({\bf x}^k) - \omega(\sigma)$
whenever $\lambda_k \geq \sigma$, where $\sigma > 0$ is fixed. 
First, we show the following property for the new iteration ${\bf x}^{k+1}$.

\begin{lemma}\label{le:decrease_lemma}
Suppose that ${\bf x}^{k+1}$ is a point generated by \eqref{eq:damped_PNM}. Then,
we have
\begin{equation}\label{eq:decrease_eq}
F({\bf x}^{k+1}) \leq F({\bf x}^k) - \left[\alpha_k\lambda_k^2 - \omega^{*}(\alpha_k\lambda_k)\right],
\end{equation}
provided that $\alpha_k\lambda_k < 1$.
\end{lemma}

\begin{proof}
Let ${\bf y}^k = {\bf x}^k + {\bf d}^k$, where $\mathbf{d}^k$ is the unique solution of \ref{eq:cvx_subprob1}. 
It follows from the optimality condition of \eqref{eq:cvx_subprob_optimality} that there exists ${\bf v}_k \in\partial{g}({\bf y}^k)$ such that
\begin{equation}\label{eq:lm2_est0}
{\bf v}_k = -\nabla{f}({\bf x}^k) - \nabla^2f({\bf x}^k)({\bf y}^k - {\bf x}^k).
\end{equation}
Since $f$ is self-concordant, by \eqref{eq:SC_bound2}, for any ${\bf x}^{k+1}$ such that $\lambda_k = \norm{{\bf x}^{k+1} - {\bf x}^k}_{{\bf x}^k} < 1$ we have
\begin{align}\label{eq:lm2_est1}
F({\bf x}^{k+1}) &\leq F({\bf x}^k) + \nabla{f}({\bf x}^k)^T({\bf x}^{k+1}\!\! - {\bf x}^k) \\
& + \omega^{*}(\norm{{\bf x}^{k+1}\!\! -{\bf x}^k}_{{\bf x}^k}) + g({\bf x}^{k+1}) - g({\bf x}^k). \nonumber 
\end{align}
Since $g$ is convex, $\alpha\in [0,1]$, by using \eqref{eq:lm2_est0} we have
\begin{align}\label{eq:lm2_est2}
g({\bf x}^{k+1}) - g({\bf x}^k) &= g((1-\alpha_k){\bf x}^k + \alpha_k{\bf y}^k) - g({\bf x}^k) \nonumber\\
&\leq \alpha_k[g({\bf y}^k) \!-\! g({\bf x}^k)] \nonumber\\
& \leq \alpha {\bf v}_k^T({\bf y}^k \!-\! {\bf x}^k) \\
& = \alpha_k{\bf v}_k^T{\bf d}^k \nonumber\\
&\overset{\tiny\eqref{eq:lm2_est0}}{=} -\alpha_k\nabla{f}({\bf x}^k)^T{\bf d}^k - \alpha_k\norm{{\bf d}^k}_{\bar{\bf x}^k}^2.\nonumber
\end{align}
Now, substituting \eqref{eq:lm2_est2} into \eqref{eq:lm2_est1} and noting that ${\bf x}^{k+1} - {\bf x}^k = \alpha_k{\bf d}^k$ we obtain the following result
\begin{align}\label{eq:lm2_est3}
F({\bf x}^{k\!+\!1}) &\leq F({\bf x}^k) \!+\! \omega^{*}(\alpha\norm{{\bf d}^k}_{{\bf x}^k}) \!-\! \alpha_k({\bf d}^k)^T\nabla^2f({\bf x}^k){\bf d}^k\nonumber\\
&= F({\bf x}^k) - \left[\alpha_k\lambda_k^2 - \omega^{*}(\alpha_k\lambda_k)\right],
\end{align}
which is indeed \eqref{eq:decrease_eq}, provided that $\alpha_k\lambda_k < 1$.
\end{proof}

The following theorem provides an explicit formula for the step size $\alpha_k$.

\begin{theorem}\label{th:choose_alpha}
Let $\mathbf{x}^{k+1}$ be a new point generated by the scheme \eqref{eq:damped_PNM} and $\lambda_k$ be defined by
\eqref{eq:PNT_decrement}. Then, if we choose $\alpha_k := (1 + \lambda_k)^{-1} \in (0, 1]$ then
\begin{equation}\label{eq:decrease_eq2}
F({\bf x}^{k+1}) \leq F({\bf x}^k) - \omega(\lambda_k).
\end{equation}
Moreover, the step $\alpha_k = (1 + \lambda_k)^{-1}$ is optimal.
\end{theorem}

\begin{proof}
By the choice of $\alpha_k$, we have $\alpha_k\lambda_k  =  (1+\lambda_k)^{-1}\lambda_k  < 1$. By using the estimate \eqref{eq:decrease_eq} we have
\begin{equation*}
F({\bf x}^{k+1}) \leq F({\bf x}^k) - (1+\lambda_k)^{-1}\lambda_k^2 + \omega^{*}\left((1+\lambda_k)^{-1}\lambda_k\right). 
\end{equation*}
Since $\frac{t^2}{1+t} - \omega^{*}(\frac{t}{1+t}) = \omega(t)$ for any $t > 0$, the last inequality implies \eqref{eq:decrease_eq2}.
Finally, we note that the function $\varphi(\alpha) := \alpha\lambda(1+\lambda) + \ln(1-\alpha\lambda)$ is maximized at $\alpha_k = (1+\lambda_k)^{-1}$, showing that $\alpha_k$ is optimal.
\end{proof}

Theorem \ref{th:choose_alpha} shows that the damped proximal Newton scheme generates a new point ${\bf x}^{k+1}$ that decreases
$F$ of \eqref{eq:min_Fx} at least $\omega(\sigma)$ at each iteration, whenever $\lambda_k \geq \sigma$. 

\paragraph{Quadratic convergence:}
Similar to the full-step proximal-Newton scheme \eqref{eq:FPNM}, we can also show the quadratic convergence of the damped proximal-Newton scheme
\eqref{eq:damped_PNM}. This statement is summarized in the following theorem.

\begin{theorem}\label{th:quad_converg_DPNM}
For a given $\mathbf{x}^k \in\dom{F}$, let $\mathbf{x}^{k+1}$ be a new point generated by the scheme \eqref{eq:damped_PNM}  with $\alpha_k :=
(1+\lambda_k)^{-1}$.
Then, if $\lambda_k < 1-\frac{1}{\sqrt{2}}$, it holds that 
\begin{equation}\label{eq:DPNM_estimate}
\lambda_{k+1} \leq 2(1 - 2\lambda_k - \lambda_k^2)^{-1}\lambda^2_k.
\end{equation}
Hence, the sequence $\set{{\bf x}^k}_{k\geq 0}$ generated by \eqref{eq:damped_PNM} with $\alpha_k = (1+\lambda_k)^{-1}$ starting from $\mathbf{x}^0\in\dom{F}$
such that $\lambda_0 \leq \sigma \leq \bar{\sigma} := \sqrt{5} - 2 \approx 0.236068$ locally converges to ${\bf x}^{\ast}$,
the unique solution of \eqref{eq:min_Fx} at a quadratic rate.
\end{theorem}
The proof of Theorem \ref{th:quad_converg_DPNM} can be found in the supplementary document.
Note that the value $\bar{\sigma}$ in Theorem \ref{th:quad_converg_DPNM} is larger than in Theorem \ref{th:quad_converg_FPNM}. However, both values are not
tight.

\subsection{The algorithm pseudocode}\label{subsec:PNM_alg}
As proved by Theorems \ref{th:choose_alpha} and \ref{th:quad_converg_DPNM}, we can use the damped proximal-Newton scheme to build the algorithm. 
Now, we present a two-phase proximal-Newton algorithm. We first select a constant $\sigma \in (0, \bar{\sigma}]$. At each
iteration, we compute the new point ${\bf x}^{k+1}$ by using the damped proximal Newton scheme \eqref{eq:damped_PNM} until we get $\lambda_k \leq
\sigma$. Then,  we switch to the full-step Newton scheme and perform it until the convergence is achieved. 
These steps are described in Algorithm \ref{alg:PNM}.

\begin{algorithm}[!ht]\caption{(\textit{Proximal Newton algorithm})}\label{alg:PNM}
\begin{algorithmic}
   \STATE {\bfseries Initialization:} 
   \STATE Require a starting point ${\bf x}^0\in\dom{F}$ and a constant $\sigma \in (0, \bar{\sigma}]$, 
          where $\bar{\sigma} := \frac{(5-\sqrt{17})}{4}\approx 0.219224$. 
   \STATE\textbf{Phase 1:}~(\textit{Damped proximal Newton iterations}).
   \FOR{$j=0$ {\bfseries to} $j_{\max}$}
   \STATE 1. Compute the proximal-Newton search direction ${\bf d}^j$ as the unique solution of $\mathcal{Q}(\mathbf{x}^j)$.
   \STATE 2. Compute $\lambda_j := \norm{{\bf d}^j}_{{\bf x}^j}$.
   \STATE 3. \textbf{if}~$\lambda_j \leq \sigma$ \textbf{then} terminate Phase 1.
   \STATE 4. Otherwise, update the next iteration ${\bf x}^{j+1} := {\bf x}^j + \alpha_j{\bf d}^j$, where $\alpha_j := (1+\lambda_j)^{-1}$.
   \ENDFOR
   \STATE \textbf{Phase 2:}~(\textit{Full-step proximal Newton iterations}).
   \STATE Set ${\bf x}^0 := {\bf x}^j$ from Phase 1 and choose a desired accuracy $\varepsilon > 0$.
   \FOR{$k=0$ {\bfseries to} $k_{\max}$}
   \STATE 1. Compute the proximal-Newton direction $\mathbf{d}^{k}$ as the unique solution of \ref{eq:cvx_subprob1}.
   \STATE 2. Compute $\lambda_k := \norm{{\bf d}^k}_{{\bf x}^k}$.
   \STATE 3. \textbf{if}~$\lambda_k \leq \varepsilon$ \textbf{then} terminate Phase 2.
   \STATE 4. Otherwise, update  ${\bf x}^{k+1} := {\bf x}^k + {\bf d}^k$.
   \ENDFOR
\end{algorithmic}
\end{algorithm}

Note that the radius $\sigma$ of the quadratic convergence region in Algorithm \ref{alg:PNM} can be fixed at its upper bound $\bar{\sigma}$. The maximum number
of iterations $j_{\max}$ and $k_{\max}$ can also be specified, if necessary.

\subsection{Iteration-complexity analysis}\label{subsec:complexity}
We analyze the complexity of Algorithm \ref{alg:PNM} by separating Phase 1 and Phase 2. This analysis is summarized in the following theorem.

\begin{theorem}\label{th:complexity}
The maximum number of iterations required in Phase 1 does not exceed $j_{\max} :=
\left\lfloor\frac{F({\bf x}^0)-F({\bf x}^{\ast})}{\omega(\sigma)}\right\rfloor + 1$, where ${\bf x}^{\ast}$ is the unique solution of \eqref{eq:min_Fx}.
The maximum number of iterations required in Phase 2 to obtain $\lambda_k \leq \varepsilon$ does not exceed $k_{\max} :=
O\left(\ln\ln\left(\frac{c}{\varepsilon}\right)\right)$, where $c := (1 - 4\sigma + 2\sigma^2)^{-1} > 0$.
\end{theorem}

\begin{proof}
Since $\lambda_j \geq \sigma$ for all $j\geq 0$ in Phase 1, it follows from Theorem \ref{th:choose_alpha} that $F({\bf x}^{j+1}) \leq F({\bf x}^j) -
\omega(\sigma)$.
By induction we have $F({\bf x}^{\ast}) \leq F({\bf x}^{j_{\max}}) \leq F({\bf x}^0) - j_{\max}\omega(\sigma)$. This implies that $j_{\max} \leq [F({\bf x}^0) -
F({\bf x}^{\ast})]/\omega(\sigma)$. Hence, we can fix 
\begin{equation*}
j_{\max} := \left\lfloor\frac{F({\bf x}^0)-F({\bf x}^{\ast})}{\omega(\sigma)}\right\rfloor + 1.
\end{equation*}
Let $c := (1 - 4\sigma + 2\sigma^2)^{-1} > 0$. 
By induction, it follows from Theorem \ref{th:quad_converg_FPNM} that we have $\lambda_k \leq
\left(c\right)^{2^k-1}\lambda_0^{2^k} \leq \left(c\right)^{2^k-1}\sigma^{2^k}$. In order to ensure
$\lambda_k \leq\varepsilon$, we require $\left(c\right)^{2^k-1}\sigma^{2^k} \leq \varepsilon$, which leads to $k \leq
O\left(\ln\ln(c/\varepsilon)\right)$. Hence, we can show that $k_{\max} := O\left(\ln\ln(c/\varepsilon)\right)$.
\end{proof}

We note that we do not use $j_{\max}$ as a stopping criterion of Phase 1 of Algorithm \ref{alg:PNM}. 
In practice, we only need an upper bound of this quantity. 
If we fix $\sigma$ at $\bar{\sigma}$ then
$c\approx 4.561553$ and the complexity of Phase 2 becomes $O\left(\ln\ln\left(\frac{4.5}{\varepsilon}\right)\right)$.

\subsection{Dual solution approach of the subproblem}\label{subsec:solve_subprob}
In this subsection we consider a specific instance of $g$: $g({\bf x}) := \rho\norm{{\bf x}}_1$.
First, we derive a dual formulation of the convex subproblem \ref{eq:cvx_subprob1}. For notational convenience, we let ${\bf q}_k := \nabla{f}({\bf x}^k)$,
${\bf H}_k := \nabla^2{f}({\bf x}^k)$. Then, the convex subproblem \ref{eq:cvx_subprob1} can be written equivalently as
\begin{equation}\label{eq:cvx_subprob2}
\min_{{\bf y}\in\mathbb{R}^n}\set{\frac{1}{2}{\bf y}^T{\bf H_ky} + ({\bf q}_k - {\bf H}_k{\bf x}^k)^T{\bf y} + \rho\norm{{\bf y}}_1}.
\end{equation}
By using the min-max principle, we can write \eqref{eq:cvx_subprob2} as
\begin{equation}\label{eq:min_max_subprob}
\max_{\norm{{\bf u}}_{\infty}\!\leq 1}\min_{{\bf y}\in\mathbb{R}^n}\set{\frac{1}{2}{\bf y}^T{\bf H}_k{\bf y} \!+\! ({\bf q} \!-\! {\bf H}_k{\bf x}^k)^T{\bf
y} \!+\! \rho {\bf u}^T{\bf y}}. 
\end{equation}
Solving the inner minimization in \eqref{eq:min_max_subprob} we obtain:
\begin{equation}\label{eq:cvx_dual_subprob2}
\min_{\norm{{\bf u}}_{\infty}\leq 1}\set{\frac{1}{2}{\bf u}^T{\bf H}_k^{-1}{\bf u} + \tilde{{\bf q}}_k^T{\bf u}},
\end{equation}
where $\tilde{{\bf q}}_k := \frac{1}{\rho}({\bf H}_k^{-1}{\bf q}_k - {\bf x}^k)$.
Note that the objective function $\varphi_k({\bf u}) := \frac{1}{2}{\bf u}^T{\bf H}_k^{-1}{\bf u} + \tilde{{\bf q}}_k^T{\bf u}$ of \eqref{eq:cvx_dual_subprob2}
is strongly convex. One can apply the fast projected gradient methods with linear convergence rate in \cite{Nesterov2007,Beck2009} for solving this problem.

In order to recover the solution of the primal subproblem \ref{eq:cvx_subprob1}, we note that the solution of the parametric minimization problem in
\eqref{eq:min_max_subprob} is given by ${\bf y}^{*}_k({\bf u}) := {\bf x}^k - {\bf H}_k^{-1}({\bf q}^k + \rho {\bf u})$.
Let ${\bf u}^{*}_k$ be the optimal solution of \eqref{eq:cvx_dual_subprob2}. We can recover the primal proximal-Newton search direction ${\bf d}^k$
of the subproblem \ref{eq:cvx_subprob1} as
\begin{equation}\label{eq:primal_sol_of_cvx_subprob2}
{\bf d}^k = - \nabla^2f({\bf x}^k)^{-1}[\nabla{f}({\bf x}^k) + \rho {\bf u}^{*}_k].
\end{equation}
To compute the quantity $\lambda_k := \norm{{\bf d}^k}_{{\bf x}^k}$ in Algorithm \ref{alg:PNM}, we use \eqref{eq:primal_sol_of_cvx_subprob2} such that
\begin{equation}\label{eq:r_norm}
\lambda_k = \norm{\nabla{f}({\bf x}^k) + \rho {\bf u}^{*}_k}^{*}_{{\bf x}^k}. 
\end{equation}
Note that computing $\lambda_k$ by \eqref{eq:r_norm} requires the inverse of the Hessian matrix $\nabla^2{f}({\bf x}^k)$.
\section{Application to graph selection}\label{sec:app_graph_select}
In this section, we customize the theory framework of Algorithm \ref{alg:PNM} by using only Phase 1 to solve the graph selection problem \eqref{eq:glearn_prob}.

\noindent\textbf{Quantification:}
For clarity, we retain the matrix variable form as presented in \eqref{eq:glearn_prob}. We note that $f(\T)$ is a self-concordant convex function, while
$g(\T)$ is a proper, lower semicontinuous and nonsmooth convex function. 
Thus, our theory presented above can be applied to \eqref{eq:glearn_prob}. 
Given the current estimate $\T_i \succ 0$, we have $\nabla{f}(\T_i) = \widehat{\cov} - \T_i^{-1}$ and $\nabla^2{f}(\T_i) = \T_i^{-1}\kron\T_i^{-1}$.
Under this setting, the dual subproblem \eqref{eq:cvx_dual_subprob2} becomes:
\begin{equation}\label{eq:cvx_dual_subprob2_app}
{\bf U}^{\ast}_i = \argmin_{\norm{\vec({\bf U})}_{\infty}\leq 1}\set{\frac{1}{2}\trace((\T_i {\bf U})^2) + \trace(\widetilde{\bf Q}_i{\bf U})},
\end{equation}
where $\widetilde{\bf Q}_i := \rho^{-1}[\T_i\widehat{\cov}\T_i - 2\T_i]$. 
Given the dual solution ${\bf U}^{\ast}_i$ of \eqref{eq:cvx_dual_subprob2_app}, the primal proximal-Newton search direction (i.e. the solution of
\ref{eq:cvx_subprob1}) is computed as
\begin{equation}\label{eq:primal_sol_of_cvx_subprob2_app}
\boldsymbol{\Delta}_i := -\left((\T_i\widehat{\cov} - \mathbb{I})\T_i + \rho\T_i {\bf U}_i^{\ast}\T_i\right).  
\end{equation} 
The quantity $\lambda_i$ defined in \eqref{eq:r_norm} can be computed by
\begin{equation}\label{eq:r_norm_app}
\lambda_i  := \left(p - 2\cdot \trace\left({\bf W}_i\right) + \trace\left( {\bf W}_i^2\right)\right)^{1/2}.
\end{equation} where ${\bf W}_i := \T_i(\widehat{\cov} + \rho {\bf U}_i^{\ast})$. 

\noindent\textbf{The graph learning algorithm:}
Algorithm \ref{alg:PNA_for_graph_select} summarizes the proposed scheme for graph selection.

\vskip-0.1cm
\begin{algorithm}[!ht]\caption{(\textit{Dual PN for graph selection (DPNGS)})}\label{alg:PNA_for_graph_select}
\begin{algorithmic}   
   \STATE\textbf{Input: } Matrix $\Sigma\succ 0$ and a given tolerance $\varepsilon > 0$.
   \STATE\textbf{Output: } An approximate solution $\T_i$ of \eqref{eq:glearn_prob}.
   \STATE {\bfseries Initialization:} Find a starting point $\T_0\succ 0$.
   \FOR{$i=0$ {\bfseries to} $i_{\max}$}
      \STATE 1. Set $\widetilde{\bf Q}_i := \rho^{-1}\left(\T_i\widehat{\cov}\T_i - 2\T_i\right)$.
      \STATE 2. Compute ${\bf U}^{\ast}_i$ in \eqref{eq:cvx_dual_subprob2_app}.
      \STATE 3. Compute $\lambda_i$ by \eqref{eq:r_norm_app}, where $\textbf{W}_i \!:=\! \T_i(\widehat{\cov} \!+\! \rho {\bf U}_i^{\ast})$. 
      \STATE 4. If $\lambda_i \leq \varepsilon$ terminate.
      \STATE 5. Compute $\small{\boldsymbol{\Delta}_i := -\left((\T_i\widehat{\cov} - \mathbb{I})\T_i + \rho\T_i{\bf U}^{\ast}_i\T_i\right)}$.
      \STATE 6. Set $\alpha_i := (1 + \lambda_i)^{-1}$.
      \STATE 7. Update $\T_{i+1} := \T_i + \alpha_i\boldsymbol{\Delta}_i$.
   \ENDFOR
\end{algorithmic}
\end{algorithm}
Overall, Algorithm \ref{alg:PNA_for_graph_select} does not require any matrix inversion operation. It only needs matrix-vector and matrix-matrix
calculations, making the parallelization of the code easier. 
We note that due to the predefined step-size selection $\alpha_i$ in Algorithm \ref{alg:PNM} we do not need to do any backtracking line-search step. This
advantage can avoid some overhead computation regarding the evaluation of the objective function which is usually expensive in this application.

\noindent\textbf{Arithmetical complexity analysis:}
Since the analytical complexity is provided in Theorem \ref{th:complexity}, we only analyze the arithmetical complexity of Algorithm \ref{alg:PNA_for_graph_select} here. As we work through the dual problem, the primal solution is dense even if majority of the entries are rather small (e.g., smaller than $10^{-6}$).\footnote{In our MATLAB code, we made no attempts for sparsification of the primal solution. The overall complexity of the algorithm can be improved via thresholding tricks.} Hence, the arithmetical complexity of Algorithm
\ref{alg:PNA_for_graph_select} is dominated by the complexity of $p\times p$ matrix multiplications.   

For instance, the computation of $\widetilde{\bf Q}_i$ and $\boldsymbol{\Delta}_i$ require basic matrix multiplications. 
For the computation of $\lambda_i$, we require two trace operations: $\tr({\bf W}_i)$ in $\mathcal{O}(p)$ time-complexity and $\tr({\bf W}_i^2)$ in
$\mathcal{O}(p^2)$ time-complexity. We note here that, while ${\bf W}_i$ is a {\it dense} matrix, the trace operation requires only the computation of the
diagonal elements of ${\bf W}_i^2$. 
Given $\T_i$, $\alpha_i$ and $\boldsymbol{\Delta}_i$, $\T_{i+1}$ requires $O(p^2)$ time-complexity. 

To compute \eqref{eq:cvx_dual_subprob2_app}, we can use the fast projected gradient method (FPGM) \cite{Nesterov2007,Beck2009} with step size $1/L$   where $L$
is the Lipschitz constant of the gradient of the objective function in \eqref{eq:cvx_dual_subprob2_app}. 
It is easy to observe that $L_i := \gamma_{\max}^2(\T_i)$ where $\gamma_{\max}(\T_i)$ is the largest eigenvalue of $\T_i$. For sparse $\T_i$, we can
approximately compute $\gamma_{\max}(\T_i)$ is $\mathcal{O}(p^2)$ by using \textit{iterative power methods} (typically, 10 iterations suffice). 
The projection onto $\norm{\text{\textbf{vec}}({\bf U})}_\infty \leq 1$ clips the elements by unity in $\mathcal{O}(p^2)$ time. Thus, 
the time overhead due to acceleration is within $\mathcal{O}(p^2)$. 

Given the above, FPGM requires a constant number of iterations $k_{\max}$, which is independent of the dimension $p$, to achieve an $\varepsilon_{\mathrm{in}}$
solution accuracy. 
Overall, the time-complexity for the solution in  \eqref{eq:cvx_dual_subprob2_app} is $\mathcal{O}(k_{\max}M)$,  where $M$ is the cost of matrix multiplication.

\begin{remark}[Parallel and distributed implementation ability]\label{re:par_dis_impl_ability}
In Algorithm \ref{alg:PNA_for_graph_select}, the outer loop does not require any Cholesky decomposition or matrix inversion. 
Suppose that the fast projected gradient method is applied to solve the dual subproblem \eqref{eq:cvx_dual_subprob2_app}. The main operation needed in the
whole algorithm is matrix-matrix multiplication of the form $\T_i\bf{U}\T_i$, where $\T_i$ and $\bf{U}$ are symmetric positive definite. This operation can
naturally be computed in a parallel or distributed manner. For more details, we refer the reader to Chapter 1 in \cite{Bertsekas1989}.
\end{remark}
\vspace{-3mm}
\section{Numerical experiments}\label{sec:num_test}
In this section we test DPNGS (Algorithm \ref{alg:PNA_for_graph_select}  in Section \ref{sec:app_graph_select}) and compare it with the state-of-the-art graph selection algorithm QUadratic Inverse
Covariance (QUIC) algorithm \cite{Hsieh2011} on a real world data set. The QUIC algorithm is also a Newton-based method, which in addition exploits the sparsity in solving its primal subproblems. We note that QUIC was implemented in C  while our codes in this work are implemented in MATLAB.

\noindent\textbf{Implementation details:}
We test DPNGS on MATLAB 2011b running on a PC Intel Xeon X5690 at 3.47GHz per core with 94Gb RAM. 
To solve \eqref{eq:cvx_dual_subprob2_app}, we use the FPGM scheme as detailed in the supplementary material.
We terminate FPGM if either $\norm{{\bf U}_{k+1}-{\bf U}_k}_{\mathrm{F}} \leq \varepsilon_{\mathrm{in}}\max\{\norm{{\bf U}_k}_{\mathrm{F}}, 1\}$ or the number of iterations
reaches $k_{\max}$ where $\varepsilon_{\mathrm{in}} > 0$ and $k_{\max}$ will be specified later.
The stopping criterion of the outer loop is $\lambda_i \leq 10^{-6}$ and the maximum number of outer iterations is chosen as $i_{\max} := 200$. 
We test the following three variants of DPNGS: DPNGS [$\varepsilon_{\mathrm{in}} = 10^{-6}$ and $k_{\max} = 1000$], DPNGS($5$)  [$\varepsilon_{\mathrm{in}} = 10^{-4}$ and $k_{\max} = 5$], and DPNGS($10$) [$\varepsilon_{\mathrm{in}} = 10^{-5}$ and $k_{\max} = 10$].
The DPNGS(5) and DPNGS(10) variants can be considered as inexact variants of DPNGS. 

\noindent\textbf{Real-world data:}
In our experiments, we use the real biology data preprocessed by \cite{Li2010} to compare the performance of the DPNGS variants above and QUIC \cite{Hsieh2011} for $5$ problems: \texttt{Lymph} ($p=587$), \texttt{Estrogen}
($p=692$), \texttt{Arabidopsis} ($p=834$), \texttt{Leukemia} ($p=1225$) and \texttt{Hereditary} ($p=1869$).
This dataset can be found at \url{http://ima.umn.edu/~maxxa007/send_SICS/}.

\noindent\textbf{Convergence behaviour analysis:}
First, we verify the convergence behaviour of Algorithm \ref{alg:PNA_for_graph_select} by analyzing the quadratic convergence of the quantity $\lambda_i$,
where $\lambda_i$ is defined by \eqref{eq:r_norm_app}.  Our analysis is based on the \texttt{Lymph} problem with $p = 587$ variables.
We note that $\lambda_i$ reveals the weighted norm of the proximal-gradient mapping of the problem. 
The convergence behaviour is plotted in Figure \ref{fig:lambda_conv} for three different values of $\rho$, namely $\rho = 0.25$, $\rho = 0.1$, $\rho = 0.05$
and $\rho = 0.01$.
\begin{figure}[ht]
\vskip-0.05in
\begin{center}
\centerline{\includegraphics[width=8cm, height=4.4cm]{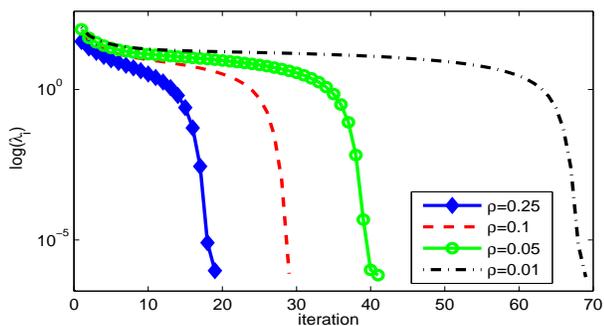}}
\caption{Quadratic convergence of DPNGS}\label{fig:lambda_conv}
\end{center}
\vskip -0.3in
\end{figure} 
Figure \ref{fig:lambda_conv} shows that whenever the values of $\lambda_i$ gets into the quadratic region, it converges with only a few iterations.
As $\rho$ becomes smaller, we need more iterations to get into the quadratic convergence region. 

Next, we illustrate the step-size $\alpha_i$ of DPNGS. Figure \ref{fig:step_conv} shows the increasing behaviour of the step size on the same dataset.
Since $\alpha_i = (1+\lambda_i)^{-1}$, it converges quickly at the last iterations.
\begin{figure}[ht]
\vskip -0.05in
\begin{center}
\centerline{\includegraphics[width=8cm, height=4.4cm]{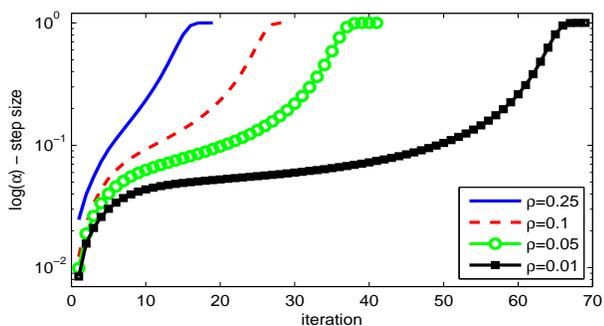}}
\caption{The step size of DPNGS}\label{fig:step_conv}
\end{center}
\vskip -0.3in
\end{figure} 
We also compare the objective value decrement of both algorithms in $y$-log-scale in Figure \ref{fig:obj_value}. 
Using the same tolerance level, we reach the objective value
$-4.141662\times10^2$ after $69$ iterations while QUIC needs $159$ iterations. 
Moreover, Figure 3 shows the quadratic convergence of our approach in contrast to QUIC; the latter requires many more iterations to slightly improve the objective. 
\begin{figure}[ht]
\vskip-0.05in
\begin{center}
\centerline{\includegraphics[width=8cm, height=4.5cm]{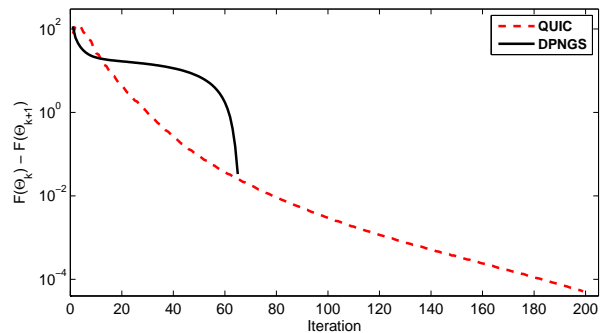}}
\caption{The difference of the objective values of DPNGS and QUIC in y-log-scale}\label{fig:obj_value}
\end{center}
\vskip -0.3in
\end{figure} 
Figure 4 is the histogram of the solution in $\log$ scale reported by DPNGS and QUIC. Due to the dual solution approach, DPNGS reports an approximate
solution with similar sparsity pattern as the one of QUIC. However, our solution has many small numbers instead of zero as in QUIC as revealed in Figure
\ref{fig:hist_sol}. \emph{This seems to be the main weakness of the dual approach}:  it obviates matrix inversions by avoiding the primal problem, which can return solutions with exact zeros thanks to its soft-thresholding prox-operator. 

As a result, DPNGS carries around extremely small coefficients (almost of them smaller than $5\times 10^{-5}$) often preventing it from achieving the same objective level as the numerical experiments on the full data set shows. At the same time, since the approach does not rely on coordinate descent on active sets, it appears much less sensitive to the choice of $\rho$. This could be an advantage of DPNGS in applications requiring smaller $\rho$ values. If exact sparsity is needed, then a single primal iteration suffices to remove the small coefficients.  

\begin{figure}[ht]
\vskip-0.05in
\begin{center}
\centerline{\includegraphics[width=8.4cm, height=3.3cm]{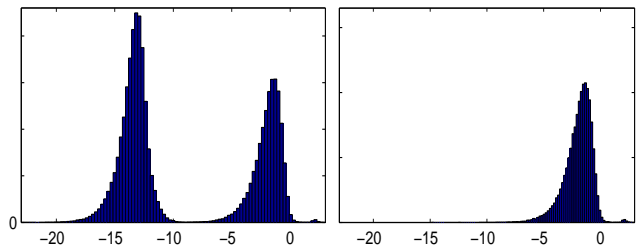}}
\caption{The histogram of the coefficient absolute values of the solution in $\log$-scale of DPNGS and QUIC (right).}\label{fig:hist_sol}
\end{center}
\vskip -0.3in
\end{figure} 

\noindent\textbf{Numerical experiments on the full dataset:}
We now report the numerical experiments on the biology dataset and compare the methods. We test both algorithms with four different values of $\rho$, namely $\rho =
0.25$, $\rho = 0.1$, $\rho = 0.05$ and $\rho = 0.01$.  The numerical results are reported in Table \ref{tbl:realdata_table}.
Since QUIC exceeds the maximum number of iterations $i_{\max} = 200$ and takes a long time, we do not report the results corresponding to $\rho = 0.01$.
We note again that at the time of this ICML submission, our implementation is done in MATLAB, while QUIC code is (carefully!) implemented in C. Hence, our timings may improve.
\begin{table*}
\caption{Summary of comparison results on real world datasets.} \label{tbl:realdata_table}
\newcommand{\cell}[1]{{\!\!\!\!}#1{\!\!}}
\newcommand{\cellbf}[1]{{\!\!\!}\textbf{#1}{\!\!}}
\begin{center}
\begin{footnotesize}
\begin{tabular}{l|rrr|rrr|rrr|rrr}\hline
\multicolumn{1}{l}{\cell{Algorithm}} & \multicolumn{3}{|c|}{\cell{$\rho=0.25$}} & \multicolumn{3}{|c|}{\cell{$\rho=0.1$}} &
\multicolumn{3}{|c|}{\cell{$\rho=0.05$}} & \multicolumn{3}{|c}{\cell{$\rho=0.01$}} \\ \hline
                                         & \cell{\#iter} & \cell{time[s]} & \cell{$F(\T_i)$} & \cell{\#iter} & \cell{time[s]} & \cell{$F(\T_i)$}
					  & \cell{\#iter} & \cell{time[s]} & \cell{$F(\T_i)$} & \cell{\#iter} & \cell{time[s]} & \cell{$F(\T_i)$} \\ \hline
\multicolumn{1}{l|}{} & \multicolumn{12}{c}{\textbf{{\color{blue}Lymph Problem} ~($p = 587$)}} \\ \hline
DPNGS                                     & \cell{19} & \cell{49.028} & \cell{613.26} & \cell{29} & \cell{61.548} & \cell{341.89}
                                         & \cell{40} & \cell{66.635} & \cell{133.60} & \cell{69} & \cell{104.259} & \cell{-414.17} \\   
DPNGS(10)                                 & \cell{39} & \cell{7.470} & \cell{613.42} & \cell{34} & \cell{8.257} & \cell{342.12}
                                         & \cell{43} & \cell{8.678} & \cell{133.87} & \cell{78} & \cell{35.543} & \cell{-413.82}\\
DPNGS(5)		                         & \cell{61} & \cell{7.067} & \cell{615.87} & \cell{30} & \cell{4.323} & \cell{344.72}
                                         & \cell{41} & \cell{5.862} & \cell{136.37} & \cell{69} & \cell{123.552} & \cell{-414.17} \\
QUIC~(C~code)                            & \cell{22} & \cell{8.392} & \cell{613.25} & \cell{44} & \cell{33.202} & \cell{341.88}
                                         & \cell{82} & \cell{176.135} & \cell{133.60} & \cell{201} & \cell{2103.788} & \cell{-414.17} \\ \hline
\multicolumn{1}{l|}{} & \multicolumn{12}{c}{\textbf{{\color{blue}Estrogen Problem} ~($p = 692$)}} \\ \hline
DPNGS                                     & \cell{24} & \cell{141.027} & \cell{627.87} & \cell{39} & \cell{171.721} & \cell{251.20}
                                         & \cell{52} & \cell{167.460} & \cell{-11.59} & \cell{83} & \cell{205.262} & \cell{-643.21} \\
DPNGS(10)                                 & \cell{56} & \cell{15.500} & \cell{628.10} & \cell{49} & \cell{14.092} & \cell{251.52}
                                         & \cell{59} & \cell{19.262} & \cell{-11.25} & \cell{90} & \cell{28.930} & \cell{-642.85} \\
DPNGS(5)                                  & \cell{39} & \cell{9.310} & \cell{631.53} & \cell{46} & \cell{8.388} & \cell{254.61} 
                                         & \cell{51} & \cell{9.332} & \cell{-7.69}  & \cell{81} & \cell{42.955} & \cell{-639.54} \\
QUIC~(C~code)                            & \cell{19} & \cell{7.060} & \cell{627.85} & \cell{43} & \cell{49.235} & \cell{251.19}
                                         & \cell{81} & \cell{244.242} & \cell{-11.60} &  \cell{-} & \cell{-} & \cell{-} \\ \hline
\multicolumn{1}{l|}{} & \multicolumn{12}{c}{\textbf{{\color{blue}Arabidopsis Problem} ~($p = 834$)}} \\ \hline
DPNGS                                     & \cell{26} & \cell{174.947} & \cell{728.57} & \cell{43} & \cell{220.365} & \cell{228.16}
                                         & \cell{61} & \cell{253.180} & \cell{-146.10} & \cell{100} & \cell{430.505} & \cell{-1086.57} \\
DPNGS(10)                                 & \cell{48} & \cell{22.268} & \cell{728.96} & \cell{45} & \cell{22.404} & \cell{228.57}
                                         & \cell{60} & \cell{26.007} & \cell{-145.72} & \cell{200} & \cell{101.428} & \cell{-1038.60} \\
DPNGS(5)                                  & \cell{38} & \cell{9.826} & \cell{733.67} & \cell{44} & \cell{11.113} & \cell{233.04}
					  & \cell{57} & \cell{18.378} & \cell{-141.84} & \cell{95} & \cell{73.948} & \cell{-1083.53} \\
QUIC~(C~code)                            & \cell{21} & \cell{19.684} & \cell{728.52} & \cell{49} & \cell{116.016} & \cell{228.14}
                                         & \cell{95} & \cell{562.532} & \cell{-146.13} & \cell{-} & \cell{-} & \cell{-} \\ \hline
\multicolumn{1}{l|}{} & \multicolumn{12}{c}{\textbf{{\color{blue}Leukemia Problem} ~($p = 1255$)}} \\ \hline
DPNGS                                     & \cell{28} & \cell{669.548} & \cell{1143.79}  & \cell{48} & \cell{624.145} & \cell{386.37}
                                         & \cell{71} & \cell{726.688} & \cell{-279.93} & \cell{130} & \cell{1398.133} & \cell{-2071.33} \\
DPNGS(10)                                 & \cell{65} & \cell{82.497} & \cell{1144.66} & \cell{48} & \cell{60.108} & \cell{387.26}
                                         & \cell{68} & \cell{84.017} & \cell{-279.12} & \cell{126} & \cell{166.567} & \cell{-2070.02} \\
DPNGS(5)                                  & \cell{49} & \cell{38.317} & \cell{1154.13} & \cell{48} & \cell{37.273} & \cell{395.08}
                                         & \cell{70} & \cell{50.886} & \cell{-271.01} & \cell{124} & \cell{258.090} & \cell{-2060.25} \\
QUIC~(C~code)                            & \cell{18} & \cell{69.826} & \cell{1143.76} & \cell{41} & \cell{344.199} & \cell{386.33}
                                         & \cell{76} & \cell{1385.577} & \cell{-280.07} & \cell{-} & \cell{-} & \cell{-} \\ \hline
\multicolumn{1}{l|}{} & \multicolumn{12}{c}{\textbf{{\color{blue}Hereditary Problem} ~($p = 1869$)}} \\ \hline
DPNGS                                     & \cell{41} & \cell{2645.875} & \cell{1258.31} & \cell{82} & \cell{3805.608} & \cell{-348.49} 
                                         & \cell{113} & \cell{5445.974} & \cell{-1609.59} & \cell{183} & \cell{9020.237} & \cell{-4569.85} \\
DPNGS(10)                                 & \cell{63} & \cell{242.528} & \cell{1261.15}  & \cell{80} & \cell{297.131} & \cell{-345.47}
					  & \cell{126} & \cell{435.159} & \cell{-1606.67} & \cell{190} & \cell{732.802} & \cell{-4566.66} \\
DPNGS(5)                                  & \cell{58} & \cell{129.821} & \cell{1290.34} & \cell{79} & \cell{169.817} & \cell{-313.87}
					  & \cell{126} & \cell{439.386} & \cell{-1606.67} & \cell{179} & \cell{1140.932} & \cell{-4537.95} \\
QUIC~(C~code)                            & \cell{21} & \cell{437.252} & \cell{1258.00} & \cell{45} & \cell{1197.895} & \cell{-348.80}
					  & \cell{84} & \cell{3182.211} & \cell{-1609.92} & \cell{-} & \cell{-} & \cell{-}  \\ \hline
\end{tabular}
\end{footnotesize}
\end{center}
\vskip -0.2in
\end{table*}

We highlight several interesting results from Table \ref{tbl:realdata_table}. First, QUIC obtains the highest accuracy results in most cases, which we
attribute to the ``lack of soft thresholding'' in our algorithm. As the DPNGS algorithm carries around a score of extremely small numbers (effectively making
the solution dense in the numerical sense), its solutions are close to QUIC's solutions within numerical precision. Moreover, QUIC is extremely efficient when the
$\rho$ value is large, since it exploits the sparsity of the putative solutions via coordinate descent. Unsurprisingly, QUIC slows down significantly as $\rho$
is decreased. 

DPGNS(5) and DPNGS(10) can obtain near optimal solutions quite rapidly. In particular, DPNGS(10) seems to be the most competitive across the board, 
often taking a fraction of QUIC's time to provide a very close solution. Hence, one can expect these schemes to be used for initializing
other algorithms. For instance, QUIC can be a good candidate. We observed in all cases that, in the first few iterations, QUIC performs several
Cholesky decompositions to stay within the positive definite cone. As the complexity of such operation is large, our step-size selection within QUIC or a DPNGS(10)
initialization can be helpful.


\vspace{-3mm}
\section{Conclusions}\label{sec:conclude}
In this paper, we present the new composite self-concordant optimization framework. As a concrete application, we demonstrate that graph learning is possible without any Cholesky decompositions via analytic step-size selection as well as without matrix inversions via a careful dual formulation within our framework. By exploiting the self-concordance in the composite graph learning
objective,  we provide an optimal step-size for this class of composite minimization with proximal Newton methods. We show that within the dual formulation
of the Newton subproblem, we do not need to explicitly calculate the gradient as it appears in a multiplication form with the Hessian. Thanks to the special structure of this multiplication, we avoid matrix inversions in graph learning. Overall, we expect our optimization framework to have more applications in signal processing/machine learning and be amenable to various parallelization techniques, beyond the ones considered in the graph learning problem.
%

\begin{thebibliography}{18}
\providecommand{\natexlab}[1]{#1}
\providecommand{\url}[1]{\texttt{#1}}
\expandafter\ifx\csname urlstyle\endcsname\relax
  \providecommand{\doi}[1]{doi: #1}\else
  \providecommand{\doi}{doi: \begingroup \urlstyle{rm}\Url}\fi

\bibitem[Banerjee et~al.(2008)Banerjee, El~Ghaoui, and
  d'Aspremont]{Banerjee2008}
Banerjee, O., El~Ghaoui, L., and d'Aspremont, A.
\newblock Model selection through sparse maximum likelihood estimation for
  multivariate gaussian or binary data.
\newblock \emph{The Journal of Machine Learning Research}, 9:\penalty0
  485--516, 2008.

\bibitem[Beck \& Teboulle(2009)Beck and Teboulle]{Beck2009}
Beck, A. and Teboulle, M.
\newblock {A} {F}ast {I}terative {S}hrinkage-{T}hresholding {A}lgorithm for
  {L}inear {I}nverse {P}roblems.
\newblock \emph{SIAM J. Imaging Sciences}, 2\penalty0 (1):\penalty0 183--202,
  2009.

\bibitem[Bertsekas \& Tsitsiklis(1989)Bertsekas and Tsitsiklis]{Bertsekas1989}
Bertsekas, D.P. and Tsitsiklis, J.~N.
\newblock \emph{{P}arallel and distributed computation: {N}umerical methods}.
\newblock Prentice Hall, 1989.

\bibitem[Boyd \& Vandenberghe(2004)Boyd and Vandenberghe]{Boyd2004}
Boyd, S. and Vandenberghe, L.
\newblock \emph{{C}onvex {O}ptimization}.
\newblock University {P}ress, Cambridge, 2004.

\bibitem[Dempster(1972)]{Dempster1972}
Dempster, A.~P.
\newblock Covariance selection.
\newblock \emph{Biometrics}, 28:\penalty0 157--175, 1972.

\bibitem[Hsieh et~al.(2011)Hsieh, Sustik, Dhillon, and Ravikumar]{Hsieh2011}
Hsieh, C.~J., Sustik, M.A., Dhillon, I.S., and Ravikumar, P.
\newblock Sparse inverse covariance matrix estimation using quadratic
  approximation.
\newblock \emph{Advances in Neutral Information Processing Systems (NIPS)},
  24:\penalty0 1--18, 2011.

\bibitem[Lee et~al.(2012)Lee, Sun, and Saunders]{Lee2012}
Lee, J.D., Sun, Y., and Saunders, M.A.
\newblock Proximal newton-type methods for convex optimization.
\newblock \emph{Tech. Report.}, pp.\  1--25, 2012.

\bibitem[Li \& Toh(2010)Li and Toh]{Li2010}
Li, L. and Toh, K.C.
\newblock An inexact interior point method for l 1-regularized sparse
  covariance selection.
\newblock \emph{Mathematical Programming Computation}, 2\penalty0 (3):\penalty0
  291--315, 2010.

\bibitem[Lu(2010)]{Lu2010}
Lu, Z.
\newblock Adaptive first-order methods for general sparse inverse covariance
  selection.
\newblock \emph{SIAM Journal on Matrix Analysis and Applications}, 31\penalty0
  (4):\penalty0 2000--2016, 2010.

\bibitem[Nesterov(2004)]{Nesterov2004}
Nesterov, Y.
\newblock \emph{{I}ntroductory lectures on convex optimization: a basic
  course}, volume~87 of \emph{Applied Optimization}.
\newblock Kluwer Academic Publishers, 2004.

\bibitem[Nesterov(2007)]{Nesterov2007}
Nesterov, Y.
\newblock Gradient methods for minimizing composite objective function.
\newblock \emph{CORE Discussion paper}, 76, 2007.

\bibitem[Nesterov \& Nemirovski(1994)Nesterov and Nemirovski]{Nesterov1994}
Nesterov, Y. and Nemirovski, A.
\newblock \emph{Interior-point Polynomial Algorithms in Convex Programming}.
\newblock Society for Industrial Mathematics, 1994.

\bibitem[Olsen et~al.(2012)Olsen, Oztoprak, Nocedal, and Rennie]{Olsen2012}
Olsen, P.A., Oztoprak, F., Nocedal, J., and Rennie, S.J.
\newblock Newton-like methods for sparse inverse covariance estimation.
\newblock \emph{Optimization Online}, 2012.

\bibitem[Ravikumar et~al.(2011)Ravikumar, Wainwright, Raskutti, and
  Yu]{Ravikumar2011}
Ravikumar, P., Wainwright, M.~J., Raskutti, G., and Yu, B.
\newblock High-dimensional covariance estimation by minimizing l1-penalized
  log-determinant divergence.
\newblock \emph{Electron. J. Statist.}, 5:\penalty0 935--988, 2011.

\bibitem[Rolfs et~al.(2012)Rolfs, Rajaratnam, Guillot, Wong, and
  Maleki]{Rolfs2012}
Rolfs, B., Rajaratnam, B., Guillot, D., Wong, I., and Maleki, A.
\newblock Iterative thresholding algorithm for sparse inverse covariance
  estimation.
\newblock In \emph{Advances in Neural Information Processing Systems 25}, pp.\
  1583--1591, 2012.

\bibitem[Scheinberg \& Rish(2009)Scheinberg and Rish]{Scheinberg2009}
Scheinberg, K. and Rish, I.
\newblock Sinco-a greedy coordinate ascent method for sparse inverse covariance
  selection problem.
\newblock \emph{preprint}, 2009.

\bibitem[Scheinberg et~al.(2010)Scheinberg, Ma, and Goldfarb]{Scheinberg2010}
Scheinberg, K., Ma, S., and Goldfarb, D.
\newblock Sparse inverse covariance selection via alternating linearization
  methods.
\newblock \emph{arXiv preprint arXiv:1011.0097}, 2010.

\bibitem[Yuan(2012)]{Yuan2012}
Yuan, X.
\newblock Alternating direction method for covariance selection models.
\newblock \emph{Journal of Scientific Computing}, 51\penalty0 (2):\penalty0
  261--273, 2012.

\end{thebibliography}

\newpage
\appendix
\icmltitle{Supplementary material}
\appendix
\section{The proofs of technical statements}\label{app:proofs}

\subsection{The proof of Theorem 3.2}
\begin{proof}
Let $\mathbf{x}^k\in\dom{F}$, we define
\begin{align}
& P_k^g := (\nabla^2f(\textbf{x}^k) + \partial{g})^{-1}, \nonumber\\
&S_k({\bf z}) := \nabla^2f(\mathbf{x}^k){\bf z} - \nabla{f}({\bf z}).\nonumber
\end{align} 
and
\begin{align*}
\mathbf{e}_k \equiv {\bf e}_k(\mathbf{x}) := [\nabla^2f(\textbf{x}^k) - \nabla^2f(\textbf{x})]\mathbf{d}^k). \nonumber
\end{align*} 
It follows from the optimality condition (7) in the main text that 
\begin{align}
{\bf 0} \in \partial{g}(\textbf{x}^{k+1}) + \nabla{f}(\textbf{x}^k) + \nabla^2f(\textbf{x}^k)(\textbf{x}^{k+1}-\textbf{x}^k). \nonumber
\end{align} 
This condition can be written equivalently to
\begin{align}
S_k(\mathbf{x}^k) + \mathbf{e}_k(\mathbf{x}^k) \in \nabla^2f(\textbf{x}^k)\textbf{x}^{k+1} + \partial{g}(\textbf{x}^{k+1}). \nonumber
\end{align} 
Therefore, the last relation leads to
\begin{align}\label{eq:th31_proof2}
\textbf{x}^{k+1} = P_k^g(S_k(\textbf{x}^k) + {\bf e}_k). 
\end{align} 
If we define $\mathbf{d}^k := \mathbf{x}^{k+1} - \mathbf{x}^k$ then
\begin{align}
\mathbf{d}^k = P_k^g(S_k(\textbf{x}^k) + {\bf e}_k) - \mathbf{x}^k.\nonumber
\end{align} 
Consequently, we also have
\begin{align}\label{eq:th31_proof3}
\mathbf{d}_{k+1} = P_k^g(S_k(\mathbf{x}^{k+1}) + \mathbf{e}_{k+1}) - \mathbf{x}^{k+1}.
\end{align}
We consider the norm $\lambda_k^1 := \norm{\textbf{d}^{k+1}}_{\bf{x}^k}$. 
By using the nonexpansive property of $P_k^{g}$, it follows from \eqref{eq:th31_proof2} and \eqref{eq:th31_proof3} that
\begin{align}\label{eq:th31_proof4}
\lambda_k^1 &= \norm{\mathbf{d}^{k+1}}_{\mathbf{x}^k} \nonumber\\
& = \norm{ P_k^g\left( S_k(\textbf{x}^{k+1}) + {\bf e}_{k+1}\right) - P_k^g\left(S_k(\textbf{x}^k) + {\bf e}_k\right)}_{\textbf{x}^k}
\nonumber\\
&\overset{\tiny(5)}{\leq} \norm{S_k(\textbf{x}^{k+1}) + {\bf e}_{k+1} - S_k(\textbf{x}^k) - {\bf e}_k}_{\textbf{x}^k}^{*} \nonumber\\
&\leq\norm{\nabla{f}(\textbf{x}^{k+1}) - \nabla{f}(\textbf{x}^k) - \nabla^2{f}(\textbf{x}^k)(\textbf{x}^{k+1}-\textbf{x}^k)}_{\textbf{x}^k}^{*} \nonumber\\
&+ \norm{{\bf e}_{k+1} - {\bf e}_k}_{\textbf{x}^k}^{*}\nonumber\\
&=\left[\norm{\int_0^1[\nabla^2{f}(\mathbf{x}^k_{\tau}) -
\nabla^2{f}(\textbf{x}^k)](\textbf{x}^{k+1}-\textbf{x}^k)d\tau}_{\textbf{x}^k}^{*}\right]_{[1]}\nonumber\\
& + \left[\norm{{\bf e}_{k+1} - {\bf e}_k}_{\textbf{x}^k}^{*}\right]_{[2]},
\end{align}
where $\mathbf{x}^k_{\tau} := \textbf{x}^k + \tau(\textbf{x}^{k+1}-\textbf{x}^k)$.
First, we estimate the first term in the last line of \eqref{eq:th31_proof4} which we denote by $[\cdot]_{[1]}$.
Now, we define
\begin{align}
\mathbf{M}_k := \int_0^1[\nabla^2{f}({\bf x}^k + \tau(\textbf{x}^{k+1}-\textbf{x}^k)) - \nabla^2f(\textbf{x}^k)]d\tau, \nonumber
\end{align} 
and 
\begin{align}
\mathbf{N}_k := \nabla^2f(\textbf{x}^k)^{-1/2}\mathbf{M}_k\nabla^2f(\textbf{x}^k)^{-1/2}. \nonumber
\end{align} 
Similar to the proof of Theorem 4.1.14 in (Nesterov, 2004), we can show that $\norm{{\bf N}_k} \leq
(1-\norm{\textbf{d}^k}_{\mathbf{x}^k})^{-1}\norm{\textbf{d}^k}_{\mathbf{x}^k}$. 
Combining this inequality and \eqref{eq:th31_proof4} we deduce
\begin{align}\label{eq:th2_est2}
[\cdot]_{[1]} &= \norm{\mathbf{M}_k{\bf d}^k}_{\textbf{x}^k}^{*} \!\leq\! \norm{{\bf N}_k}\norm{\textbf{d}^k}_{\bf{x}^k}\nonumber\\
& = (1 -\lambda_k)^{-1}\lambda_k^2.
\end{align}
Next, we estimate the second term of \eqref{eq:th31_proof4} which is denoted by $[\cdot]_{[2]}$. We note that ${\bf e}_k = \mathbf{e}_k(\textbf{x}^k) = 0$ and
\begin{align}
{\bf e}_{k+1} = \mathbf{e}_{k+1}(\textbf{x}^{k+1}) = [\nabla^2f(\textbf{x}^k) - \nabla^2f(\textbf{x}^{k+1})]\textbf{d}^{k+1}. \nonumber
\end{align} 
Let
\begin{align} 
\mathbf{P}_k := \nabla^2f(\textbf{x}^k)^{-1/2}[\nabla^2f(\textbf{x}^{k+1}) - \nabla^2f(\textbf{x}^k)]\nabla^2f(\textbf{x}^k)^{-1/2}. \nonumber
\end{align}
By applying Theorem 4.1.6 in (Nesterov, 2004), we can estimate $\norm{\mathbf{P}_k}$ as
\begin{align}\label{eq:norm_Q}
\norm{\mathbf{P}_k} &\leq \max\set{1 - (1-\norm{\textbf{d}^k}_{\mathbf{x}^k})^2, \frac{1}{(1-\norm{\textbf{d}^k}_{\mathbf{x}^k})^2}-1} \nonumber\\
& = \frac{2\lambda_k - \lambda_k^2}{(1-\lambda_k)^2}.
\end{align}
Therefore, from the definition of $[\cdot]_{[2]}$ we have
\begin{align}\label{eq:th31_proof5}
[\cdot]_{[2]}^2 &= [\norm{\mathbf{e}_{k+1} - \mathbf{e}_k}_{\textbf{x}^k}^{*}]^2 \nonumber\\
& = (\mathbf{e}_{k+1} - \mathbf{e}_k)^T\nabla^2f(\textbf{x}^k)^{-1}(\mathbf{e}_{k+1} - \mathbf{e}_k) \nonumber\\
& = (\mathbf{d}^{k+1})^T\nabla^2f(\textbf{x}^k)^{1/2}\mathbf{P}_k^2\nabla^2f(\textbf{x}^k)^{1/2}\mathbf{d}^{k+1}\nonumber\\
&\leq \norm{\mathbf{P}_k}^2\norm{\mathbf{d}^{k+1}}_{\mathbf{x}^k}^2.
\end{align}
By substituting \eqref{eq:norm_Q} into \eqref{eq:th31_proof5} we obtain
\begin{equation}\label{eq:th31_proof6}
[\cdot]_{[2]} \leq \frac{2\lambda_k - \lambda_k^2}{(1-\lambda_k)^2}\lambda_k^1.
\end{equation}
Substituting \eqref{eq:th2_est2} and \eqref{eq:th31_proof6} into \eqref{eq:th31_proof4} we obtain
\begin{align*}
\lambda_k^1 \leq \frac{\lambda_k^2}{1-\lambda_k} + \frac{2\lambda_k - \lambda_k^2}{(1-\lambda_k)^2}\lambda_k^1.
\end{align*}
By rearrange this inequality we obtain
\begin{align}\label{eq:th2_est5}
\lambda_k^1 \leq \left[\frac{1-\lambda_k}{1 - 4\lambda_k + 2\lambda_k^2}\right]\lambda_k^2.
\end{align}
On the other hand, by applying Theorem 4.1.6 in (Nesterov, 2004), we can easily show that
\begin{align}\label{eq:th2_est6}
\lambda_{k+1} = \norm{\textbf{d}^{k+1}}_{\textbf{x}^{k+1}} \leq \frac{\norm{\textbf{d}^{k+1}}_{\mathbf{x}^k}}{1 - \norm{\textbf{d}^k}_{\mathbf{x}^k}} =
\frac{\lambda_k^1}{1-\lambda_k}.
\end{align}
Combining \eqref{eq:th2_est5} and \eqref{eq:th2_est6} we obtain
\begin{align*}
\lambda_{k+1} \leq \frac{\lambda_k^2}{1-4\lambda_k + 2\lambda_k^2},
\end{align*}
which is (11) in the main text.
Finally, we consider the sequence $\set{\textbf{x}^k}_{k\geq 0}$ generated by (9) in the main text. From
(11) in the main text, we have
\begin{align}
\lambda_1 &\leq (1-4\lambda_0 + 2\lambda_0^2)^{-1}\lambda_0^2 \nonumber\\
& \leq (1-4\sigma + 2\sigma^2)^{-1}\sigma^2 \nonumber \\ &\leq \sigma \nonumber
\end{align} 
provided that $0 < \sigma \leq \frac{5-\sqrt{17}}{4}\approx 0.219224$. 
By induction, we can conclude that $\lambda_k \leq \beta$ for all $k\geq 0$. It follows from (11) in the main text that 
\begin{align}
\lambda_{k+1} \leq (1 - 4\sigma + 2\sigma^2)^{-1}\lambda_k^2\nonumber 
\end{align} 
for all $k$, which shows that $\set{\norm{\textbf{x}^k - \textbf{x}^{*}}_{\textbf{x}^k}}$ converges to zero at a quadratic rate.
\end{proof}

\subsection{The proof of Theorem 3.5}
\begin{proof}
First, we note that 
\begin{align}
\textbf{x}^{k+1} = \textbf{x}^k + \alpha_k\textbf{d}^k  = \textbf{x}^k + (1+\lambda_k)^{-1}\mathbf{x}^k. \nonumber
\end{align} 
Hence, we can estimate $\textbf{d}^{k+1}$ as
\begin{equation}\label{eq:th2_est52} 
\lambda_{k+1}   \leq \frac{\norm{\textbf{d}^{k+1}}_{\mathbf{x}^k}}{1 - \alpha_k\lambda_k} =
(1 + \lambda_k)\norm{\textbf{d}^{k+1}}_{\mathbf{x}^k}. 
\end{equation}
By a similar approach as the proof of Theorem 3.5, we can estimate $\norm{\mathbf{d}}_{\mathbf{x}^k}$ as
\begin{equation*}
\norm{\mathbf{d}}_{\mathbf{x}^k} \leq \frac{2\lambda_k^2}{(1+\lambda_k)(1-2\lambda_k - \lambda_k^2)}.
\end{equation*}
Combining this inequality  and \eqref{eq:th2_est52} we obtain (19) in the main text.

In order to prove the quadratic convergence, we first show that if $\lambda_k\leq \sigma$ then $\lambda_{k+1}\leq \sigma$ for all $k\geq 0$.
Indeed, we note that the function:
\begin{align}
\varphi(t) := 2t(1-2t - t^2)^{-1} \nonumber
\end{align} is increasing in $[0, 1-1/\sqrt{2}]$.
Let $\lambda_0 \leq \sigma$. From (19) we have:
\begin{align}
\lambda_1 \leq 2\sigma^2(1-2\sigma -\sigma^2). \nonumber
\end{align} Therefore, if
\begin{align}
2\sigma^2(1-2\sigma -\sigma^2)\leq \sigma, \nonumber
\end{align} then $\lambda_1\leq\sigma$. The last requirement leads to $0 < \sigma \leq \bar{\sigma} := \sqrt{5}-2 \approx 0.236068$.
From this argument, we conclude that if $\sigma \in (0, \bar{\sigma}]$ then if $\lambda_0\leq\sigma$ then $\lambda_1\leq\sigma$. By induction, we have
$\lambda_k\leq\sigma$ for $k\geq 0$. If we define
\begin{align}
c := 2(1-2\sigma-\sigma^2)^{-1} \nonumber
\end{align} then $c > 0$ and (19) implies $\lambda_{k+1}\leq
c\lambda^2$ which shows that the sequence $\{\lambda_k\}_{k\geq 0}$ locally converges to $0$ at a quadratic rate.
\end{proof}

\subsection{The proof of Lemma 2.2.}
\begin{proof}
From the self-concordance of $f$ we have:
\begin{align}
\omega(\norm{{\bf y}-{\bf x}}_{\bf{x}}) + f(\textbf{x}) + \nabla{f}(\textbf{x})^T({\bf y}-\textbf{x}) \leq f({\bf y}). \nonumber
\end{align} On the other hand, since $g$ is convex we have
\begin{align}
g({\bf y})\geq g(\textbf{x}) + {\bf v}^T({\bf y} - \textbf{x}) \nonumber
\end{align} for any ${\bf v}\in\partial{g}(\textbf{x})$. Hence, 
\begin{align}
F({\bf y}) &\geq F(\textbf{x}) + [\nabla{f}(\textbf{x}) + {\bf v}]^T({\bf y}-\textbf{x}) + \omega(\norm{{\bf y} - {\bf x}}_{{\bf x}}) \nonumber \\ 
&\geq F(\textbf{x}) - \lambda(\textbf{x})\norm{{\bf y} - {\bf x}}_{{\bf x}} + \omega(\norm{{\bf y} - {\bf x}}_{{\bf x}}), \nonumber
\end{align} where $\lambda(\textbf{x}) := \norm{\nabla{f}(\textbf{x}) +
\bf{v}}_\textbf{x}^{*}$. Let: 
\begin{align}
\mathcal{L}_F(F(\textbf{x})) := \set{ {\bf y}\in\mathbb{R}^n ~|~ F({\bf y}) \leq F(\textbf{x})} \nonumber
\end{align} be a sublevel set of $F$. For any $y\in\mathcal{L}_F(F(\textbf{x}))$ we have $F({\bf y}) \leq F(\textbf{x})$ which leads to:
\begin{align}
\lambda(\textbf{x})\norm{{\bf y} - {\bf x}}_{{\bf x}} \geq \omega(\norm{{\bf y} - {\bf x}}_{{\bf x}}) \nonumber
\end{align} due to the previous inequality. Note that $\omega$ is a convex and strictly increasing, the equation
$\lambda(\textbf{x})t = \omega(t)$ has unique
solution $\bar{t} > 0$ if $\lambda(\textbf{x}) < 1$. Therefore, for any $0\leq t \leq \bar{t}$ we have $\norm{\bf{y}-\bf{x}}_{\bf{x}} \leq \bar{t}$. This
implies that
$\mathcal{L}_F(F(\textbf{x}))$ is bounded. Hence, $\textbf{x}^{*}$ exists. The uniqueness of $\textbf{x}^{*}$ follows from the increase of $\omega$.
\end{proof}

\section{A fast projected gradient algorithm}{\label{FPGM}}
For completeness, we provide here a variant of the fast-projected gradient method for solving the dual subproblem
(25) in the main text. Let us recall that $\mathtt{clip}_r(X) := \mathrm{sign}(X)\min\{\abs{X},r\}$ (a point-wise operator). The algorithm is
presented as follows.

\vskip-0.1cm
\begin{algorithm}[!ht]\caption{(\textit{Fast-projected-gradient algorithm})}\label{alg:FPGA_subprob}
\begin{algorithmic}   
   \STATE\textbf{Input: } The current iteration $\T_i$ and a given tolerance $\varepsilon_{\mathrm{in}} > 0$.
   \STATE\textbf{Output: } An approximate solution ${\bf U}_k$ of (25) in the main text.
   \STATE {\bfseries Initialization:} Compute a Lipschitz constant $L$ and find a starting point ${\bf U}_0\succ 0$. 
   \STATE Set ${\bf V}_0 := {\bf U}_0$, $t_0 := 1$. 
   \FOR{$k=0$ {\bfseries to} $k_{\max}$}
      \STATE 1. ${\bf V}_{k\!+\!1} \!:=\! \mathtt{clip_1}\left({\bf U}_k \!-\! \frac{1}{L}\left[\T_i({\bf U}_k \!+\! \frac{1}{\rho}\hat{\Sigma})\T_i \!-\!
\frac{2}{\rho}\T_i\right]\right)$.
      \STATE 2. If $\norm{{\bf V}_{k+1}-{\bf V}_k}_{\mathrm{Fro}} \leq \varepsilon_{\mathrm{in}}\max\{1,\norm{{\bf V}_k}_{\mathrm{Fro}}\}$ then terminate.
      \STATE 3. $t_{k+1} := 0.5(1+ \sqrt{1 \!+\! 4t_k^2})$ and $\beta_k := \frac{t_k \!-\! 1}{t_{k\!+\!1}}$.
      \STATE 4. ${\bf U}_{k+1} := {\bf V}_{k+1} + \beta_k({\bf V}_{k+1}-{\bf V}_k)$.
   \ENDFOR
\end{algorithmic}
\end{algorithm}
The main operator in Algorithm \ref{alg:FPGA_subprob} is $\T_i{\bf U}_k\T_i$ at Step 2, where $\T_i$ and ${\bf U}_k$ are symmetric and $\T_i$ may be sparse.
This operator requires twice matrix-matrix multiplications.
The worst-case complexity of Algorithm \ref{alg:FPGA_subprob} is typically $O\left(\sqrt{\frac{L}{\varepsilon_{\mathrm{in}}}}\right)$ which is sublinear. If
$\mu = \lambda_{\min}(\T_i)$, the smallest eigenvalue of $\T_i$, is available, we can set $\beta_k := \frac{\sqrt{L}-\sqrt{\mu}}{\sqrt{L}+\sqrt{\mu}}$ and we
get a linear convergence rate.

\end{document}